\def\year{2018}\relax
\documentclass[letterpaper]{article} %DO NOT CHANGE THIS
\usepackage{aaai18}  %Required
\usepackage{times}  %Required
\usepackage{helvet}  %Required
\usepackage{courier}  %Required
\usepackage{url}  %Required
\usepackage{graphicx}  %Required

\usepackage{algorithm}
\usepackage[noend]{algpseudocode}
\usepackage{paralist}
\usepackage{amsmath}
\usepackage{amssymb}
\usepackage{amsthm}
\usepackage{todonotes}
\usepackage{tikz}
\usetikzlibrary{shapes,shapes.misc} %backgrounds
\usepackage[capitalize]{cleveref}

\newfloat{algorithm}{tbp}{lop}
\crefname{algorithm}{Alg.}{Algs.}
\Crefname{algorithm}{Algorithm}{Algorithms}

\theoremstyle{definition}
\newtheorem{definition}{Definition}

\newtheorem{theorem}{Theorem}

\begin{document}
\title{Fusing First-order Knowledge Compilation and \\the Lifted Junction Tree Algorithm\thanks{To appear in ``KI-18: Advances of AI'', published by Springer}}
\author{Tanya Braun \and Ralf M\"oller\\
Institute of Information Systems\\
University of L\"ubeck, L\"ubeck, Germany\\
\{braun, moeller\}@ifis.uni-luebeck.de\\
}
\maketitle
\begin{abstract}
Standard approaches for inference in probabilistic formalisms with first-order constructs include lifted variable elimination (LVE) for single queries as well as first-order knowledge compilation (FOKC) based on weighted model counting.
To handle multiple queries efficiently, the lifted junction tree algorithm (LJT) uses a first-order cluster representation of a model and LVE as a subroutine in its computations.
For certain inputs, the implementations of LVE and, as a result, LJT ground parts of a model where FOKC has a lifted run.
The purpose of this paper is to prepare LJT as a backbone for lifted inference and to use any exact inference algorithm as subroutine.
Using FOKC in LJT allows us to compute answers faster than LJT, LVE, and FOKC for certain inputs.
\end{abstract}

\noindent
AI areas such as natural language understanding and machine learning need efficient inference algorithms. %(ML)
Modeling realistic scenarios yields large probabilistic models, requiring reasoning about sets of individuals.
Lifting uses symmetries in a model to speed up reasoning with known domain objects.
%Further, especially in ML, multiple queries are common.
We study probabilistic inference in large models that exhibit symmetries with queries for probability distributions of random variables (randvars).

In the last two decades, researchers have advanced probabilistic inference significantly.
Propositional formalisms benefit from variable elimination (VE), which decomposes a model into subproblems and evaluates them in an efficient order \cite{ZhaPo94}.
Lifted VE (LVE), introduced in \cite{PooZh03} and expanded in \cite{Braz07,MilZeHaKa08,TagDa12}, saves computations by reusing intermediate results for isomorphic subproblems.
\mbox{Taghipour et al.\ }formalise LVE by defining lifting operators while decoupling the constraint language from the operators \cite{TagFiDaBl13}.
The lifted junction tree algorithm (LJT) sets up a first-order junction tree (FO jtree) to handle multiple queries efficiently \cite{BraMo16a}, using LVE as a subroutine.
LJT is based on the propositional junction tree algorithm \cite{LauSp88}, which includes a junction tree (jtree) and a reasoning algorithm for efficient handling of multiple queries.
Approximate lifted inference often uses lifting in conjunction with belief propagation \cite{SinDo08,GogDo10,AhmKeMlNa13}.
To scale lifting, \mbox{Das et al.\ }use graph databases storing compiled models to count faster \cite{DaWuKhKeNa16}.
Other areas incorporate lifting to enhance efficiency, e.g., in continuous or dynamic models \cite{ChoiAmHi12,VlaMeBrRa16}, logic programming \cite{BelLaRiCoZe14}, and theorem proving \cite{GogDo11}.

Logical methods for probabilistic inference are often based on weighted model counting (WMC) \cite{ChaDa08}.
Propositional knowledge compilation (KC) compiles a weighted model into a deterministic decomposable negation normal form (d-DNNF) circuit for probabilistic inference \cite{DarMa02}.
\mbox{Chavira and Darwiche} combine VE and KC as well as algebraic decision diagrams for local symmetries to further optimise inference runtimes \cite{ChaDa07}.
\mbox{Van den Broeck et al.\ }apply lifting to KC and WMC, introducing weighted first-order model counting (WFOMC) and a first-order d-DNNF \cite{BroTaMeDaRa11,BroDa12}, with newer work on asymmetrical models \cite{BroNi15}.

%Lifted inference sparks progress in various fields.
For certain inputs, LVE, LJT, and FOKC start to struggle either due to model structure or size.
The implementations of LVE and, as a consequence, LJT ground parts of a model if randvars of the form $Q(X), Q(Y), X \not= Y$ appear, where parameters $X$ and $Y$ have the same domain, even though in theory, LVE handles those occurrences of just-different randvars \cite{ApsBr11}.
While FOKC does not ground in the presence of such constructs in general, it can struggle if the model size increases.
The purpose of this paper is to prepare LJT as a backbone for lifted query answering (QA) to use any exact inference algorithm as a subroutine.
Using FOKC and LVE as subroutines, we fuse LJT, LVE, and FOKC to compute answers faster than LJT, LVE, and FOKC alone for the inputs described above.

The remainder of this paper is structured as follows:
First, we introduce notations and FO jtrees and recap LJT.
Then, we present conditions for subroutines of LJT, discuss how LVE works in this context and FOKC as a candidate, before fusing LJT, LVE, and FOKC.
We conclude with future work.

\section{Preliminaries}
This section introduces notations and recap LJT.
We specify a version of the smokers example (e.g., \cite{BroTaMeDaRa11}), where two friends are more likely to both smoke and smokers are more likely to have cancer or asthma.
Parameters allow for representing people, avoiding explicit randvars for each individual.

\subsubsection{Parameterised Models}
To compactly represent models with first-order constructs, parameterised models use logical variables (logvars) to parameterise randvars, abbreviated PRVs.
They are based on work by Poole \cite{Poo03}. % and use concepts like logvars and PRVs first introduced therein.
\begin{definition}
	Let $\mathbf{L}$, $\Phi$, and $\mathbf{R}$ be sets of logvar, factor, and randvar names respectively.
	A \emph{PRV} $R(L_1, \dots, L_n)$, $n \geq 0$, is a syntactical construct with $R \in \mathbf{R}$ and $L_1, \dots, L_n \in \mathbf{L}$ to represent a set of randvars. % behaving identically.
	For PRV $A$, the term $range(A)$ denotes possible values.
	A logvar $L$ has a domain $\mathcal{D}(L)$.
	A \emph{constraint} $(\mathbf{X},C_{\mathbf{X}})$ is a tuple with a sequence of logvars $\mathbf{X} = (X_1, \dots, X_n)$ and a set $C_{\mathbf{X}} \subseteq \times_{i = 1}^n\mathcal{D}(X_i)$ restricting logvars to given values.
	The symbol $\top$ marks that no restrictions apply and may be omitted.
	For some $P$, the term $lv(P)$ refers to its logvars, $rv(P)$ to its PRVs with constraints, and $gr(P)$ to all instances of $P$ grounded w.r.t.\ its constraints.
\end{definition}

For the smoker example, let $\mathbf{L} = \{X, Y\}$ and $\mathbf{R} = \{Smokes, Friends\}$ to build boolean PRVs $Smokes(X)$, $Smokes(Y)$, and $Friends(X,Y)$.
We denote $A = true$ by $a$ and $A = false$ by $\neg a$.
Both logvar domains are $\{alice, eve, bob\}$.
An inequality $X \not= Y$ yields a constraint $C = ((X,Y), \{(alice{,}eve), (alice{,}bob),\linebreak (eve{,}alice), (eve{,}bob), (bob{,}alice), (bob{,}eve)\})$.
$gr(Friends(X,Y) | C)$ refers to all propositional randvars that result from replacing $X,Y$ with the tuples in $C$.
%$gr(Friends(X,Y) | \top)$ also includes randvars for $(alice{,}alice), (eve{,}eve), (bob{,}bob)$.
Parametric factors (parfactors) combine PRVs as arguments.
A parfactor describes a function, identical for all argument groundings, that maps argument values to the reals (potentials), of which at least one is non-zero.
\begin{definition}
	Let $\mathbf{X} \subseteq \mathbf{L}$ be a set of logvars, $\mathcal{A} = (A_1, \dots, A_n)$ a sequence of PRVs, each built from $\mathbf{R}$ and possibly $\mathbf{X}$, $\phi : \times_{i = 1}^n range(A_i) \mapsto \mathbb{R}^+$ a function, $\phi \in \Phi$, and $C$ a constraint $(\mathbf{X}, C_\mathbf{X})$.
	We denote a \emph{parfactor} $g$ by $\forall \mathbf{X} : \phi(\mathcal{A}) | C$.
	We omit $(\forall \mathbf{X} :)$ if $\mathbf{X} = lv(\mathcal{A})$.
	A set of parfactors forms a \emph{model} $G := \{g_i\}_{i=1}^n$.
\end{definition}

We define a model $G_{ex}$ for the smoker example, adding the binary PRVs $Cancer(X)$ and $Asthma(X)$ to the ones above.
%With $\Phi = \{\phi_0, \phi_1, \phi_2, \phi_3, \phi_4, \phi_5\}$, 
The model reads $G_{ex} = \{g_i\}_{i=0}^5$, %$g_0 = \phi_0(Friends(X,Y), Smokes(X), Smokes(Y)) | C$, $g_1 = \phi_1(Friends(X,Y))\linebreak | C$, $g_2 = \phi_2(Smokes(X)) | \top$, $g_3 = \phi_3(Cancer(X)) | \top$, $g_4 = \phi_5(Smokes(X),\linebreak Asthma(X)) | \top$, and $g_5 = \phi_4(Smokes(X), Cancer(X)) | \top$.
\begin{align*}
	g_0 &= \phi_0(Friends(X,Y), Smokes(X), Smokes(Y)) | C,	\\
	g_1 &= \phi_1(Friends(X,Y)) | C, 			\\
	g_2 &= \phi_2(Smokes(X)) | \top, 			\\
	g_3 &= \phi_3(Cancer(X)) | \top,			\\
	g_4 &= \phi_5(Smokes(X), Asthma(X)) | \top,	\\
	g_5 &= \phi_4(Smokes(X), Cancer(X)) | \top.	
\end{align*}
$g_0$ has eight, $g_1$ to $g_3$ have two, and $g_4$ and $g_5$ four input-output pairs (omitted here). 
Constraint $C$ refers to the constraint given above. % with six factors in $gr(g_i)$, $i \in \{0,1\}$, with identical $\phi_i$.
The other constraints are $\top$.
\Cref{fig:exFG} depicts $G_{ex}$ as a graph with five variable nodes and six factor nodes for the PRVs and parfactors with edges to arguments.

%Evidence also displays symmetries if observing the same value for $n$ ground randvars of a PRV \cite{TagFiDaBl13}.
%A parfactor $g_E = \phi_E(P(\mathbf{X}))| C_E$ holds evidence for PRV $P(\mathbf{X})$.
%Potential function $\phi_E$ and constraint $C_E$ encode the observed values and randvars, respectively.
%Assume we observe the value $true$ for ten randvars of the PRV $Smokes(X)$.
%The corresponding parfactor is $g_E = \phi_E(Smokes(X))| C_E$.
%$C_E$ represents the domain of $X$ restricted to the 10 instances and $\phi_E(true) = 1$ and $\phi_E(false) = 0$.
%%Some technical remarks:
%%To \emph{absorb} evidence, we split all parfactors $g_i$ that cover $gr(P(X)|C_E)$, also called shattering \citep{BrazAmRo05}.
%%Splitting means adding a duplicate $g_i'$ and restricting $C_i$ and $C_i'$, i.e., we restrict $C_i$ to those tuples that contain $gr(P(X)|C_E)$ and $C_i'$ to the rest.
%%$g_i$ absorbs $g_E$.

The \emph{semantics} of a model $G$ is given by grounding and building a full joint distribution.
With $Z$ as the normalisation constant, $G$ represents the full joint probability distribution $P_G = \frac{1}{Z} \prod_{f \in gr(G)} f$.
The QA problem asks for a likelihood of an event, a marginal distribution of some randvars, or a conditional distribution given events, all queries boiling down to computing marginals w.r.t.\ a model's joint distribution.
Formally, $P(\mathbf{Q} | \mathbf{E})$ denotes a (conjunctive) query with $\mathbf{Q}$ a set of grounded PRVs and $\mathbf{E} = \{E_k = e_k\}_k$ a set of events (grounded PRVs with range values).
If $\mathbf{E} = \emptyset$, the query is for a conditional distribution.
A query for $G_{ex}$ is $P(Cancer(eve) | friends(eve, bob), smokes(bob))$.
We call $\mathbf{Q} = \{Q\}$ a singleton query.
Lifted QA algorithms seek to avoid grounding and building a full joint distribution.
Before looking at lifted QA, we introduce FO jtrees.

\begin{figure}
\centering
\begin{tikzpicture}[rv/.style={draw, ellipse, inner sep=1.5pt, minimum width=22mm, minimum height=6mm},pf/.style={draw, rectangle, fill=gray},label distance=0.5mm]
	\node[rv]													(SY)		{};
	\node[draw = none]											(Smy)	{$Smokes(Y)$};
	\node[pf, right of=SY, node distance=18mm, yshift=1mm, xshift=-1mm]	(XYa)	{};
	\node[pf, right of=SY, node distance=18mm, label=330:{$g_0$}]		(XY)		{};
	\node[rv, right of=XY, node distance=20mm, minimum width=26mm]		(F)	{};
	\node[draw=none, right of=XY, node distance=20mm]				(Fr)		{$Friends(X,Y)$};
	\node[pf, right of=F, node distance=17mm, yshift=1mm, xshift=-1mm]	(Fga)		{};
	\node[pf, right of=F, node distance=17mm, label=0:{$g_1$}]			(Fg)		{};
	\node[pf, right of=F, node distance=17mm, yshift=-1mm, xshift=1mm]	(Fgb)		{};
	\node[rv, below of=XY, yshift=2mm] 								(SX)		{};
	\node[draw=none, below of=XY, yshift=2mm] 						(Smx)	{$Smokes(X)$};
	\node[pf, right of=SX, node distance=15mm, yshift=1mm, xshift=-1mm]	(Sga)		{};
	\node[pf, right of=SX, node distance=15mm, label=0:{$g_2$}]			(Sg)		{};
	\node[pf, right of=SX, node distance=15mm, yshift=-1mm, xshift=1mm]	(Sgb)		{};
	\draw (SY) -- (XY);
	\draw (XY) -- (F);
	\draw (XY) -- (SX);
	\draw (Sg) -- (SX);
	\draw (Fg) -- (F);
	\node[pf, right of=SY, node distance=18mm, xshift=1mm, yshift=-1mm]	(XYb)	{};
	\node[rv, left of=SX, node distance=25mm, yshift=-6mm]				(A)		{};
	\node[draw=none, left of=SX, node distance=25mm, yshift=-6mm]		(Ax)		{$Asthma(X)$};
	\node[rv, right of=SX, node distance=25mm, yshift=-6mm]				(C)		{};
	\node[draw=none, right of=SX, node distance=25mm, yshift=-6mm]		(Cx)		{$Cancer(X)$};
	\node[pf, right of=C, node distance=14.5mm, yshift=1mm, xshift=-1mm]	(Cga)		{};
	\node[pf, right of=C, node distance=14.5mm, label=0:{$g_3$}]			(Cg)		{};
	\node[pf, right of=C, node distance=14.5mm, yshift=-1mm, xshift=1mm]	(Cgb)		{};
	\node[pf, left of=SX, xshift=0mm, yshift=-5mm]						(ASa)	{};
	\node[pf, left of=SX, label=0:{$g_4$},yshift=-6mm, xshift=1mm]			(AS)		{};
	\draw (SX) -- (AS);
	\draw (AS) -- (A);
	\draw (Cg) -- (C);
	\node[pf, left of=SX, xshift=2mm, yshift=-7mm]						(ASb) 	{};
	\node[pf, right of=SX, xshift=-2mm, yshift=-5mm]					(CSa) 	{};
	\node[pf, right of=SX, label=180:{$g_5$},yshift=-6mm, xshift=-1mm]		(CS)		{};
	\draw (SX) -- (CS);
	\draw (CS) -- (C);
	\node[pf, right of=SX, xshift=0mm, yshift=-7mm]					(CSb)	{};
\end{tikzpicture}
\caption{Parfactor graph for $G_{ex}$}
\label{fig:exFG}
\end{figure}
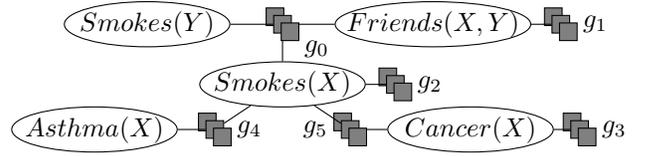
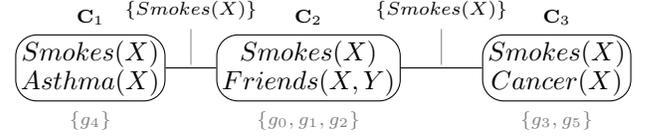
\begin{figure}
\centering
\begin{tikzpicture}[pc/.style={draw, rounded corners=8pt, align=center, node distance=29mm, minimum height=6mm, inner sep=2pt}]
	\node[pc, label={[gray]270:{\scriptsize $\{g_4\}$}},label={90:{\scriptsize $\mathbf{C}_1$}}]							(c1) {$Smokes(X)$\\ $Asthma(X)$};
	\node[pc, right of=c1, label={[gray]270:{\scriptsize $\{g_0, g_1, g_2\}$}},label={90:{\scriptsize $\mathbf{C}_2$}}]			(c2) {$Smokes(X)$\\ $Friends(X,Y)$};
	\node[pc, right of=c2, xshift=4mm, label={[gray]270:{\scriptsize $\{g_3, g_5\}$}},label={90:{\scriptsize $\mathbf{C}_3$}}]	(c3) {$Smokes(X)$\\ $Cancer(X)$};
	\draw (c1) -- node[inner sep=1pt, pin={[yshift=0mm]90:{\scriptsize $\{Smokes(X)\}$}}]	{} (c2);
	\draw (c2) -- node[inner sep=1pt, pin={[yshift=0mm]90:{\scriptsize $\{Smokes(X)\}$}}] 	{} (c3);
\end{tikzpicture}
\caption{FO jtree for $G_{ex}$ (local models in grey)}
\label{fig:fojt}
\end{figure}

\subsubsection{First-order Junction Trees}
LJT builds an FO jtree to cluster a model into submodels that contain all information for a query after propagating information.
An FO jtree, defined as follows, constitutes a lifted version of a jtree. % (cf.\ \cite{Dar09} for details).
Its nodes are parameterised clusters (parclusters), i.e., sets of PRVs connected by parfactors.
\begin{definition}
	Let $\mathbf{X}$ be a set of logvars, $\mathbf{A}$ a set of PRVs with $lv(\mathbf{A}) \subseteq \mathbf{X}$, and $C$ a constraint on $\mathbf{X}$.
	Then, $\forall \mathbf{X} {:} \mathbf{A} | C$ denotes a \emph{parcluster}.
	We omit $(\forall \mathbf{X}{:})$ if $\mathbf{X} = lv(\mathbf{A})$.
	An \emph{FO jtree} for a model $G$ is a cycle-free graph $J = (V, E)$, where $V$ is the set of nodes (parclusters) and $E$ the set of edges.
	$J$ must satisfy three properties:
	\begin{inparaenum}[(i)]
		\item $\forall \mathbf{C}_i \in V$: $\mathbf{C}_i \subseteq rv(G)$.
		\item $\forall g  \in G$: $\exists \mathbf{C}_i \in V$ s.t.\ $rv(g) \subseteq \mathbf{C}_i$.
		\item If $\exists A \in rv(G)$ s.t.\ $A \in \mathbf{C}_i \wedge A \in \mathbf{C}_j$, then $\forall \mathbf{C}_k$ on the path between $\mathbf{C}_i$ and $\mathbf{C}_j$: $A \in \mathbf{C}_k$.
	\end{inparaenum}
	The parameterised set $\mathbf{S}_{ij}$, called \emph{separator} of edge $\{i, j\} \in E$, is defined by $\mathbf{C}_i \cap \mathbf{C}_j$.
	The term $nbs(i)$ refers to the neighbours of node $i$. %, defined as $\{j | \{i,j\} \in E\}$.
	Each $\mathbf{C}_i \in V$ has a \emph{local model} $G_i$ and $\forall g \in G_i$: $rv(g) \subseteq \mathbf{C}_i$.
	The $G_i$'s partition $G$.
\end{definition}
\Cref{fig:fojt} shows an FO jtree for $G_{ex}$ with the following parclusters, %$\mathbf{C}_1 = \forall X : \{Smokes(X), Asthma(X)\} | \top$, $\mathbf{C}_2 = \forall X, Y : \{Smokes(X), Friends(X,Y)\} | C$, and $\mathbf{C}_3 = \forall X : \{Smokes(X), Cancer(X)\} | \top$.
\begin{align*}
	\mathbf{C}_1 &= \forall X : \{Smokes(X), Asthma(X)\} | \top,	\\
	\mathbf{C}_2 &= \forall X, Y : \{Smokes(X), Friends(X,Y)\} | C,	\\
	\mathbf{C}_3 &= \forall X : \{Smokes(X), Cancer(X)\} | \top.
\end{align*}
Separators are $\mathbf{S}_{12} = \mathbf{S}_{23} = \{Smokes(X)\}$.
As $Smokes(X)$ and $Smokes(Y)$ model the same randvars, $\mathbf{C}_2$ names only one.
Parfactor $g_2$ appears at $\mathbf{C}_2$ but could be in any local model as $rv(g_2) = \{Smokes(X)\} \subset \mathbf{C}_i\ \forall\ i \in \{1,2,3\}$.
%\cite{BraMo16a} details building FO jtrees.
We do not consider building FO jtrees here (cf.\ \cite{BraMo16a} for details).

\begin{algorithm}
\caption{Outline of the Lifted Junction Tree Algorithm}
\label{alg:ljt}
\begin{algorithmic}
	\Procedure{LJT}{Model $G$, Queries $\{\mathbf{Q}_j\}_{j=1}^{m}$, Ev. $\mathbf{E}$}
	\State Construct FO jtree $J$ for $G$
	\State Enter $\mathbf{E}$ into $J$
	\State Pass messages on $J$
	\For{each query $\mathbf{Q}_j$}
		\State Find subtree $J'$ for $\mathbf{Q}_j$
		\State Extract submodel $G'$ from $J'$
		\State Answer $\mathbf{Q}_j$ on $G'$
	\EndFor
	\EndProcedure
\end{algorithmic}
\end{algorithm}

\subsubsection{Lifted Junction Tree Algorithm}
LJT answers a set of queries efficiently by answering queries on smaller submodels.
\Cref{alg:ljt} outlines LJT for a set of queries (cf.\ \cite{BraMo16a} for details).
LJT starts with constructing an FO jtree.
It enters evidence for a local model to absorb whenever the evidence randvars appear in a parcluster.
Message passing propagates local information through the FO jtree in two passes:
LJT sends messages from the periphery towards the center and then back.
A message is a set of parfactors over separator PRVs.
For a message $m_{ij}$ from node $i$ to neighbour $j$, LJT eliminates all PRVs not in separator $\mathbf{S}_{ij}$ from $G_i$ and the messages from other neighbours using LVE.
Afterwards, each parcluster holds all information of the model in its local model and received messages. 
LJT answers a query by finding a subtree whose parclusters cover the query randvars, extracting a submodel of local models and outside messages, and answering the query on the submodel.
%For singleton queries, the subtree is one parcluster and the submodel its local model and received messages.
In the original LJT, LJT eliminates randvars for messages and queries using LVE.
\section{LJT as a Backbone for Lifted Inference}
LJT provides general steps for efficient QA given a set of queries.
It constructs an FO jtree and uses a subroutine to propagate information and answer queries.
To ensure a lifted algorithm run without groundings, evidence entering and message passing impose some requirements on the algorithm used as a subroutine.
After presenting those requirements, we analyse how LVE matches the requirements and to what extend FOKC can provide the same service.

\subsubsection{Requirements}
LJT has a domain-lifted complexity, meaning that if a model allows for computing a solution without grounding part of a model, LJT is able to compute the solution without groundings, i.e., has a complexity linear in the domain size of the logvars.
Given a model that allows for computing solutions without grounding part of a model, the subroutine must be able to handle message passing and query answering without grounding to maintain the domain-lifted complexity of LJT.

Evidence displays symmetries if observing the same value for $n$ instances of a PRV~\cite{TagFiDaBl13}.
Thus, for evidence handling, the algorithm needs to be able to handle a set of observations for some instances of a single PRV in a lifted way.
Calculating messages entails that the algorithm is able to calculate a form of parameterised, conjunctive query over the PRVs in the separator.
In summary, LJT requires the following:
\begin{enumerate}
	\item Given evidence in the form of a set of observations for some instances of a single PRV, the subroutine must be able to absorb the evidence independent of the size of the number of instances in the set.
	\item Given a parcluster with its local model, messages, and a separator, the subroutine must be able to eliminate all PRVs in the parcluster that do not appear in the separator in a domain-lifted way.
\end{enumerate}

The subroutine also establishes which kind of queries LJT can answer.
The expressiveness of the query language for LJT follows from the expressiveness of the inference algorithm used.
If an algorithm answers queries of single randvar, LJT answers this type of query.
If an algorithm answers maximum a posteriori (MAP) queries, the most likely assignment to a set of randvars, LJT answers MAP queries.
Next, we look at how LVE fits into LJT.

\subsubsection{Lifted Variable Elimination}
First, we take a closer look at LVE before analysing it w.r.t.\ the requirements of LJT.
To answer a query, LVE eliminates all non-query randvars.
In the process, it computes VE for one case and exponentiates its result for isomorphic instances (lifted summing out).
Taghipour implements LVE through an operator suite (see \cite{TagFiDaBl13} for details).
\Cref{alg:lqa} shows an outline.
All operators have pre- and postconditions to ensure computing a result equivalent to one for $gr(G)$.
Its main operator \emph{sum-out} realises lifted summing out.
An operator \emph{absorb} handles evidence in a lifted way.
The remaining operators (\emph{count-convert}, \emph{split}, \emph{expand}, \emph{count-normalise}, \emph{multiply}, \emph{ground-logvar}) aim at enabling lifted summing out, transforming part of a model.

LVE as a subroutine provides lifted absorption for evidence handling.
Lifted absorption splits a parfactor into one part, for which evidence exists, and one part without evidence.
The part with evidence then absorbs the evidence by absorbing it once and exponentiating the result for all isomorphic instances.
For messages, a relaxed QA routine computes answers to parameterised queries without making all instances of query logvars explicit.
LVE answers queries for a likelihood of an event, a marginal distribution of a set of randvars, and a conditional distribution of a set of randvars given events.
LJT with LVE as a subroutine answers the same queries.
Extensions to LJT or LVE enable even more query types, such as queries for a most probable explanation or MAP \cite{BraMo18a}.

\begin{algorithm}
\caption{Outlines of Lifted QA Algorithms}
\label{alg:lqa}
\begin{algorithmic}
%	\State \hspace{-3mm}\Call{LVE}{Model $G$, Query $\mathbf{Q}$, Evidence $\mathbf{E}$}
	\Function{LVE}{Model $G$, Query $\mathbf{Q}$, Evidence $\mathbf{E}$}
	\State Absorb $\mathbf{E}$ in $G$
	\While{$G$ has non-query PRVs}
		\If{PRV $A$ fulfils \emph{sum-out} preconditions}
			\State Eliminate $A$ using \emph{sum-out}
		\Else
			\State Apply transformator
		\EndIf
	\EndWhile
	\State \Return Multiply parfactors in $G$ \Comment $\alpha$-normalise
	\EndFunction
\vspace{-2.5mm}
\State\hspace{-3.5mm}\hrulefill
	\Procedure{FOKC}{Model $G$, Queries $\{Q_j\}_{j=1}^{m}$, Ev. $\mathbf{E}$}
		\State Reduce $G$ to WFOMC problem with $\Delta, w_T, w_F$
		\State Compile a circuit $\mathcal{C}_{e}$ for $\Delta$, $\mathbf{E}$
		\For{each query $Q_j$}
			\State Compile a circuit $\mathcal{C}_{qe}$ for $\Delta$, $Q_j$, $\mathbf{E}$
			\State Compute $P(Q_j | \mathbf{E})$ through WFOMCs in $\mathcal{C}_{qe}, \mathcal{C}_{e}$
		\EndFor
	\EndProcedure
\end{algorithmic}
\end{algorithm}

\subsubsection{First-order Knowledge Compilation}
FOKC aims at solving a WFOMC problem by building FO d-DNNF circuits given a query and evidence and computing WFOMCs on the circuits.
Of course, different compilation flavours exist, e.g., compiling into a low-level language \cite{KazPo16}.
But, we focus on the basic version of FOKC with an implementation available.
We briefly take a look at WFOMC problems, FO d-DNNF circuits, and QA with FOKC, before analysing FOKC w.r.t.\ the LJT requirements.
See \cite{BroTaMeDaRa11} for details.

%\paragraph{WFOMC Problem}
Let $\Delta$ be a theory of constrained clauses and $w_T$ a positive and $w_F$ a negative weight function.
Clauses follow standard notations of (function-free) first-order logic.
A constraint expresses, e.g., an (in)equality of two logvars.
$w_T$ and $w_F$ assign weights to predicates in $\Delta$.
A \emph{WFOMC problem} consists of computing
\begin{align*}
	%wmc(\Delta, w_T, w_F) = 
	\sum_{I \models \Delta} \prod_{a \in I} w_T(pred(a)) \prod_{a \in HB(T) \setminus I} w_F(pred(a))
\end{align*}
where $I$ is an interpretation of $\Delta$ that satisfies $\Delta$, $HB(T)$ is the Herbrand base and $pred$ maps atoms to their predicate.
See \cite{Bro13} for a description of how to transform parfactor models into WFOMC problems.

%To transform parfactor models into WFOMC problems, one maps each input-output pair in a parfactor to a formula with corresponding weights.
%Consider parfactor $g_0 \in G_{ex}$ as an example.
%Assume its potential function $\phi_0(Friends(X,Y), Smokes(X), Smokes(Y))$ maps the input $(true, true, true)$ to output $7.39$ and the remaining inputs to output $1$.
%Then, $g_0$ translates into two formulas,
%\begin{compactenum}[(i)]
%	\item $\forall X, Y, X \not= Y : f_1(X,Y) \Leftrightarrow\\ friends(X,Y) \wedge smokes(X) \wedge smokes(Y)$,\\
%		$w_T(f_1)$ set to $7.39$ ($w_F(f_1)$ set to $1$), and
%	\item $\forall X, Y, X \not= Y : f_2(X,Y) \Leftrightarrow\\ \neg friends(X,Y) \vee \neg smokes(X) \vee \neg smokes(Y)$,\\
%		$w_T(f_2)$ set to $1$ ($w_F(f_2)$ set to $1$),
%\end{compactenum}
%creating new predicates $f_1(X, Y)$ and $f_2(X, Y)$ for the weight functions.
%The first formula describes the first input-output pair given above.
%The second formula catches the remaining cases that all have the same weight.

FOKC converts $\Delta$ to be in FO d-DNNF, where all conjunctions are decomposable (all pairs of conjuncts independent) and all disjunctions are deterministic (only one disjunct true at a time).
The normal form allows for efficient reasoning as computing the probability of a conjunction decomposes into a product of the probabilities of its conjuncts and computing the probability of a disjunction follows from the sum of probabilities of its disjuncts.
An \emph{FO d-DNNF circuit} represents such a theory as a directed acyclic graph.
Inner nodes are labelled with $\vee$ and $\wedge$.
Additionally, set-disjunction and set-conjunction represent isomorphic parts in $\Delta$.
Leaf nodes contain atoms from $\Delta$.
The process of forming a circuit is called compilation.

Now, we look at how FOKC answers queries.
\Cref{alg:lqa} shows an outline with input model $G$, a set of query randvars $\{Q_i\}_{i=1}^m$, and evidence $\mathbf{E}$.
FOKC starts with transforming $G$ into a WFOMC problem $\Delta$ with weight functions $w_T$ and $w_F$.
It compiles a circuit $\mathcal{C}_e$ for $\Delta$ including $\mathbf{E}$.
For each query $Q_i$, FOKC compiles a circuit $\mathcal{C}_{qe}$ for $\Delta$ including $\mathbf{E}$ and $Q_i$.
It then computes 
\begin{align}
	P(Q_i | \mathbf{E}) = \frac{WFOMC(\mathcal{C}_{qe}, w_T, w_F)}{WFOMC(\mathcal{C}_{e}, w_T, w_F)}\label{eq:wfomc}
\end{align}
by propagating WFOMCs in $\mathcal{C}_{qe}$ and $\mathcal{C}_{e}$ based on $w_T$ and $w_F$.
FOKC can reuse the denominator WFOMC for all $Q_i$.

Regarding the potential of FOKC as a subroutine for LJT, FOKC does not fulfil all requirements.
FOKC can handle evidence through conditioning \cite{BroDa12}.
But, a lifted message passing is not possible in a domain-lifted and exact way without restrictions.
FOKC answers queries for a likelihood of an event, a marginal distribution of a single randvar, and a conditional distribution for a single randvar given events.
Inherently, conjunctive queries are only possible if the conjuncts are probabilistically independent \cite{DarMa02}, which is rarely the case for separators.
Otherwise, FOKC has to invest more effort to take into account that the probabilities overlap.
Thus, the restricted query language means that LJT cannot use FOKC for message calculations in general.
Given an FO jtree with singleton separators, message passing with FOKC as a subroutine may be possible.
FOKC as such takes ground queries as input or computes answers for random groundings, so FOKC for message passing needs an extension to handle parameterised queries.
FOKC may not fulfil all requirements, but we may combine LJT, LVE, and FOKC into one algorithm to answer queries for models where LJT with LVE as a subroutine struggles.

\section{Fusing LJT, LVE, and FOKC}
%Fusing LJT and FOKC overcomes the individual limitations we have identified above.
We now use LJT as a backbone and LVE and FOKC as subroutines, fusing all three algorithms.
\Cref{alg:fokcljt} shows an outline of the fused algorithm named LJTKC.
Inputs are a model $G$, a set of queries $\{Q_j\}_{j=1}^m$, and evidence $\mathbf{E}$.
Each query $Q_j$ has a single query term in contrast to a set of randvars $\mathbf{Q}_j$ in LVE and LJT.
The change stems from FOKC to ensure a correct result.
Thus, LJTKC has the same expressiveness regarding the query language as FOKC.

The first three steps of LJTKC coincide with LJT as specified in \cref{alg:lqa}:
LJTKC builds an FO jtree $J$ for $G$, enters $\mathbf{E}$ into $J$, and passes messages in $J$ using LVE for message calculations.
During evidence entering, each local model covering evidence randvars absorbs evidence.
LJTKC calculates messages based on local models with absorbed evidence, spreading the evidence information along with other local information.
After message passing, each parcluster $\mathbf{C}_i$ contains in its local model and received messages all information from $G$ and $\mathbf{E}$.
This information is sufficient to answer queries for randvars contained in $\mathbf{C}_i$ and remains valid as long as $G$ and $\mathbf{E}$ do not change.
At this point, FOKC starts to interleave with the original LJT procedure.

LJTKC continues its preprocessing.
For each parcluster $\mathbf{C}_i$, LJTKC extracts a submodel $G'$ of local model $G_i$ and all messages received and reduces $G'$ to a WFOMC problem with theory $\Delta_i$ and weight functions $w^i_F, w^i_T$.
It does not need to incorporate $\mathbf{E}$ as the information from $\mathbf{E}$ is contained in $G'$ through evidence entering and message passing.
LJTKC compiles an FO d-DNNF circuit $\mathcal{C}_i$ for $\Delta_i$ and computes a WFOMC $c_i$ on $\mathcal{C}_i$.
In precomputing a WFOMC $c_i$ for each parcluster, LJTKC uses that the denominator of \cref{eq:wfomc} is identical for varying queries on the same model and evidence.
For each query handled at $\mathbf{C}_i$, the submodel consists of $G'$, resulting in the same circuit $\mathcal{C}_i$ and WFOMC $c_i$.

\begin{algorithm}
\caption{Outline of LJTKC}
\label{alg:fokcljt}
\begin{algorithmic}
	\Procedure{LJTKC}{Model $G$, Queries $\{Q_j\}_{j=1}^{m}$, Evidence $\mathbf{E}$}
		\State Construct FO jtree $J$ for $G$
		\State Enter $\mathbf{E}$ into $J$
		\State Pass messages on $J$ \Comment LVE as subroutine
		\For{each parcluster $\mathbf{C}_i$ of $J$ with local model $G_i$}
			\State Form submodel $G' \gets G_i \cup \bigcup_{j \in nbs(i)} m_{ij}$
			\State Reduce $G'$ to WFOMC problem with $\Delta_i, w^i_T, w^i_F$
			\State Compile a circuit $\mathcal{C}_i$ for $\Delta_i$
			\State Compute $c_i = WFOMC(\mathcal{C}_i, w^i_T, w^i_F)$
		\EndFor
		\For{each query ${Q}_j$}
			\State Find parcluster $\mathbf{C}_i$ where $Q_j \in \mathbf{C}_i$
			\State Compile a circuit $\mathcal{C}_{q}$ for $\Delta_i$, $Q_j$
			\State Compute $c_q = WFOMC(\mathcal{C}_q, w^i_T, w^i_F)$
			\State Compute $P(Q_j | \mathbf{E}) = {c_q}/{c_i}$
		\EndFor
	\EndProcedure
\end{algorithmic}
\end{algorithm}

To answer a query $Q_j$, LJTKC finds a parcluster $\mathbf{C}_i$ that covers $Q_j$ and compiles an FO d-DNNF circuit $\mathcal{C}_q$ for $\Delta_i$ and $Q_j$.
It computes a WFOMC $c_q$ in $\mathcal{C}_q$ and determines an answer to $P(Q_j | \mathbf{E})$ by dividing the just computed WFOMC $c_q$ by the precomputed WFOMC $c_i$ of this parcluster.
LJTKC reuses $\Delta_i$, $w^i_T$, and $w^i_F$ from preprocessing.

\subsubsection{Example Run}
For $G_{ex}$, LJTKC builds an FO jtree as depicted in \cref{fig:fojt}.
Without evidence, message passing commences.
LJTKC sends messages from parclusters $\mathbf{C}_1$ and $\mathbf{C}_3$ to parcluster $\mathbf{C}_2$ and back.
For message $m_{12}$ from $\mathbf{C}_1$ to $\mathbf{C}_2$, LJTKC eliminates $Asthma(X)$ from $G_1$ using LVE.
For message $m_{32}$ from $\mathbf{C}_3$ to $\mathbf{C}_2$, LJTKC eliminates $Cancer(X)$ from $G_3$ using LVE.
%Both eliminations mean for LVE to perform one \emph{sum-out} operation each.
For the messages back, LJTKC eliminates $Friends(X, Y)$ each time, for message $m_{21}$ to $\mathbf{C}_1$ from $G_2 \cup m_{32}$ and for message $m_{23}$ to $\mathbf{C}_3$ from $G_2 \cup m_{12}$.
Each parcluster holds all model information encoded in its local model and received messages, which form the submodels for the compilation steps.
At $\mathbf{C}_1$, the submodel contains $G_1 = \{g_4\}$ and $m_{21}$.
At $\mathbf{C}_2$, the submodel contains $G_2 = \{g_0, g_1, g_2\}$, $m_{12}$, and $m_{32}$.
At $\mathbf{C}_3$, the submodel contains $G_3 = \{g_3, g_5\}$ and $m_{23}$.

For each parcluster, LJTKC reduces the submodel to a WFOMC problem, compiles a circuit for the problem specification, and computes a parcluster WFOMC.
Given, e.g., query randvar $Cancer(eve)$, LJTKC takes a parcluster that contains the query randvar, here $\mathbf{C}_3$.
It compiles a circuit for the query and $\Delta_3$, computes a query WFOMC $c_q$, and divides $c_q$ by $c_3$ to determine $P(cancer(eve))$.
Next, we argue why QA with LJTKC is sound.

\begin{theorem}
	LJTKC is sound, i.e., computes a correct result for a query $Q$ given a model $G$ and evidence $\mathbf{E}$.
\end{theorem}
\begin{proof}
	We assume that LJT is correct, yielding an FO jtree $J$ for model $G$, which means, $J$ fulfils the three junction tree properties, which allows for local computations based on \cite{SheSh90}.
	Further, we assume that LVE is correct, ensuring correct computations for evidence entering and message passing, and that FOKC is correct, computing correct answers for single term queries.
	
	LJTKC starts with the first three steps of LJT.
	It constructs an FO jtree for $G$, allowing for local computations.
	Then, LJTKC enters $\mathbf{E}$ and calculates messages using LVE, which produces correct results given LVE is correct. 
	After message passing, each parcluster holds all information from $G$ and $\mathbf{E}$ in its local model and received messages, which allows for answering queries for randvars that the parcluster contains.
	At this point, the FOKC part takes over, taking all information present at a parcluster and compiling a circuit and computing a WFOMC, which produces correct results given FOKC is correct.
	The same holds for the compilation and computations done for query $Q$.
	Thus, LJTKC computes a correct result for $Q$ given $G$ and $\mathbf{E}$.
\end{proof}

\begin{figure*}
\begin{minipage}{.48\textwidth}
\centering
	\includegraphics[width=\textwidth]{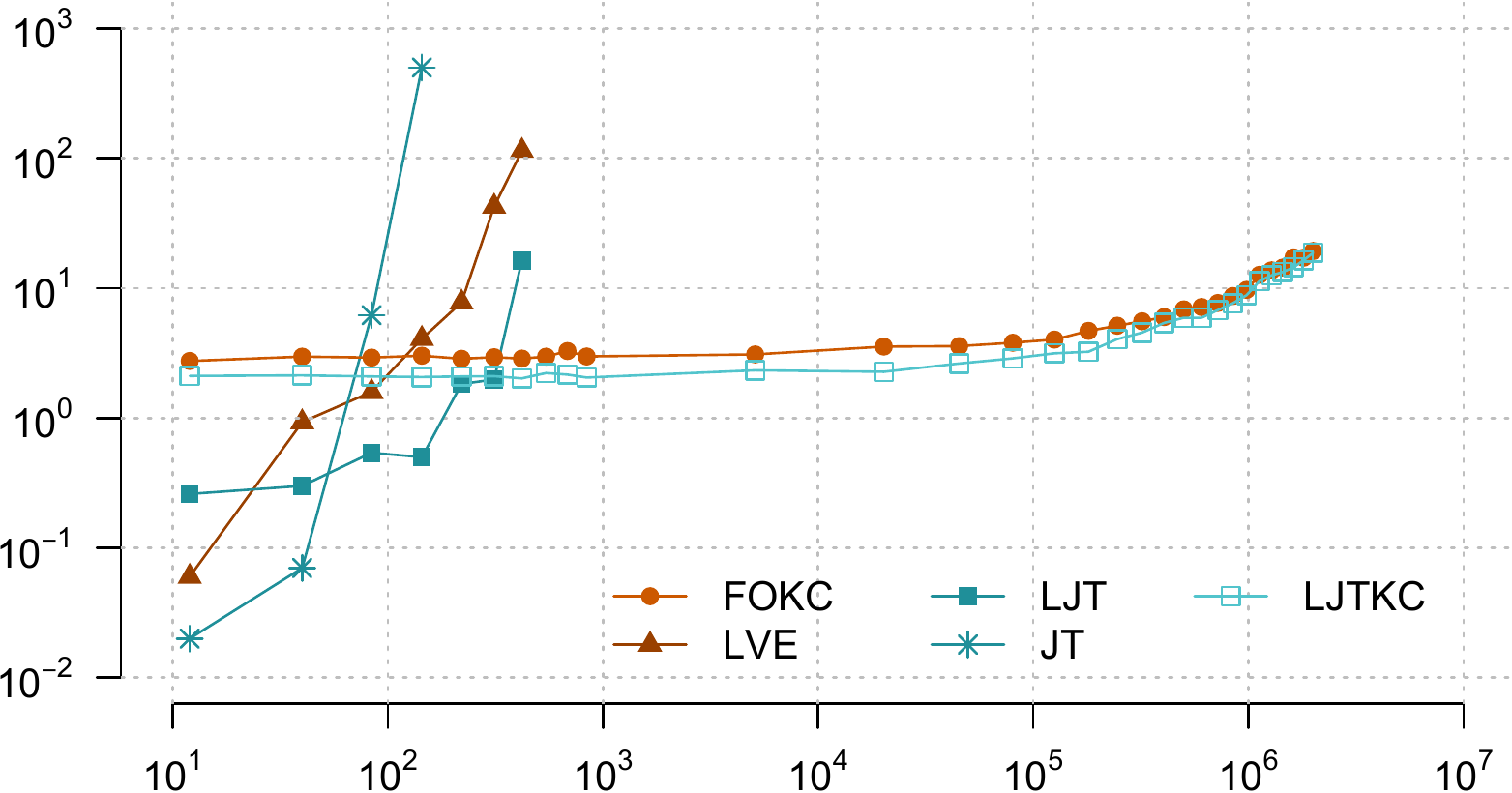}
\caption{Runtimes [ms] for $G_{ex}$; on x-axis: $|gr(G_{ex})|$ from $12$ to $2{,}002{,}000$; both axes on log scale; points connected for readability}
\label{fig:evS}
\end{minipage}\hfill
\begin{minipage}{.48\textwidth}
\centering
	\includegraphics[width=\textwidth]{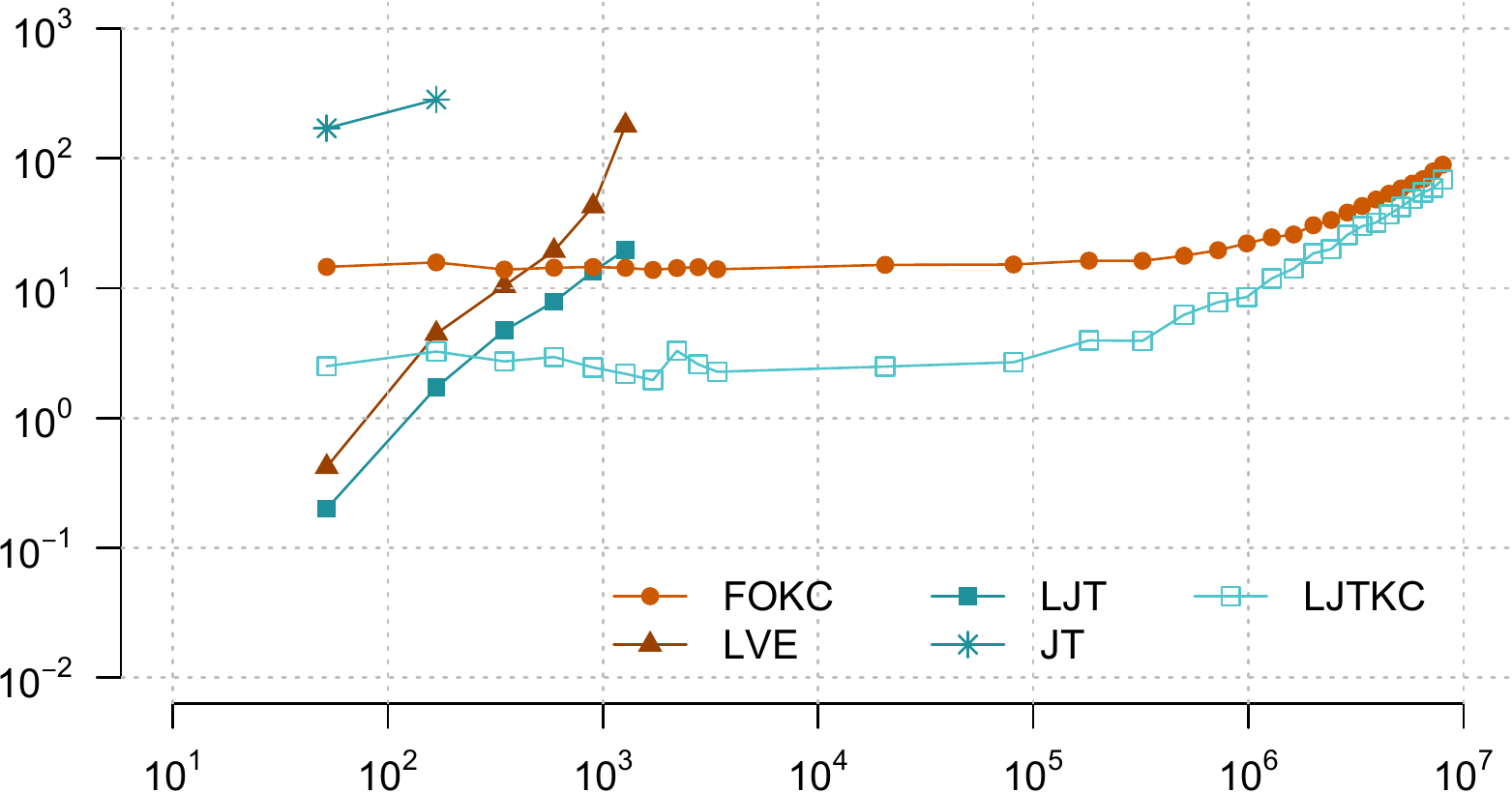}
\caption{Runtimes [ms] for $G_l$; on x-axis: $|gr(G_l)|$ from $52$ to $8{,}010{,}000$; both axes on log scale; points connected for readability}
\label{fig:evL}
\end{minipage}
\end{figure*}

\subsubsection{Theoretical Discussion}
We discuss space and runtime performance of LJT, LVE, FOKC, and LJTKC in comparison with each other.

LJT requires \emph{space} for its FO jtree as well as storing the messages at each parcluster, while FOKC takes up space for storing its circuits.
As a combination of LJT and FOKC, LJTKC stores the preprocessing information produced by both LJT and FOKC.
Next to the FO jtree structure and messages, LJTKC stores a WFOMC problem specification and a circuit for each parcluster.
%If LJTKC does not use the messages again after forming a WFOMC problem, it could drop the messages to free up space.
Since the implementation of LVE for the $X \not= Y$ cases causes LVE (and LJT) to ground, the space requirements during QA are increasing with rising domain sizes.
Since LJTKC avoids the groundings using FOKC, the space requirements during QA are smaller than for LJT alone.
W.r.t.\ circuits, LJTKC stores more circuits than FOKC but the individual circuits are smaller and do not require conditioning, which leads to a significant blow-up for the circuits.

LJTKC accomplishes speeding up QA for certain challenging inputs by fusing LJT, LVE, and FOKC.
The new algorithm has a faster runtime than LJT, LVE, and FOKC as it is able to precompute reusable parts and provide smaller models for answering a specific query through the underlying FO jtree with its messages and parcluster compilation.
In comparison with FOKC, LJTKC speeds up runtimes as answering queries works with smaller models.
In comparison with LJT and LVE, LJTKC is faster when avoiding groundings in LVE.
Instead of precompiling each parcluster, which adds to its overhead before starting with answering queries, LJTKC could compile on demand.
On-demand compilation means less runtime and space required in advance but more time per initial query at a parcluster.
One could further optimise LJTKC by speeding up internal computations in LVE or FOKC (e.g., caching for message calculations or pruning circuits using context-specific information)

In terms of \emph{complexity}, LVE and FOKC have a time complexity linear in terms of the domain sizes of the model logvars for models that allow for a lifted solution.
LJT with LVE as a subroutine also has a time complexity linear in terms of the domain sizes for query answering.
For message passing, a factor of $n$, which is the number of parclusters, multiplies into the complexity, which basically is the same time complexity as answering a single query with LVE.
LJTKC has the same time complexity as LJT for message passing since the algorithms coincide.
For query answering, the complexity is determined by the FOKC complexity, which is linear in terms of domain sizes.
Therefore, LJTKC has a time complexity linear in terms of the domain sizes.
Even though, the original LVE and LJT implementations show a practical problem in translating the theory into an efficient program, the worst case complexity for liftable models is linear in terms of domain sizes.

The next section presents an empirical evaluation, showing how LJTKC speeds up QA compared to FOKC and LJT for challenging inputs.

\begin{figure*}
\begin{minipage}{.48\textwidth}
\centering
	\includegraphics[width=\columnwidth]{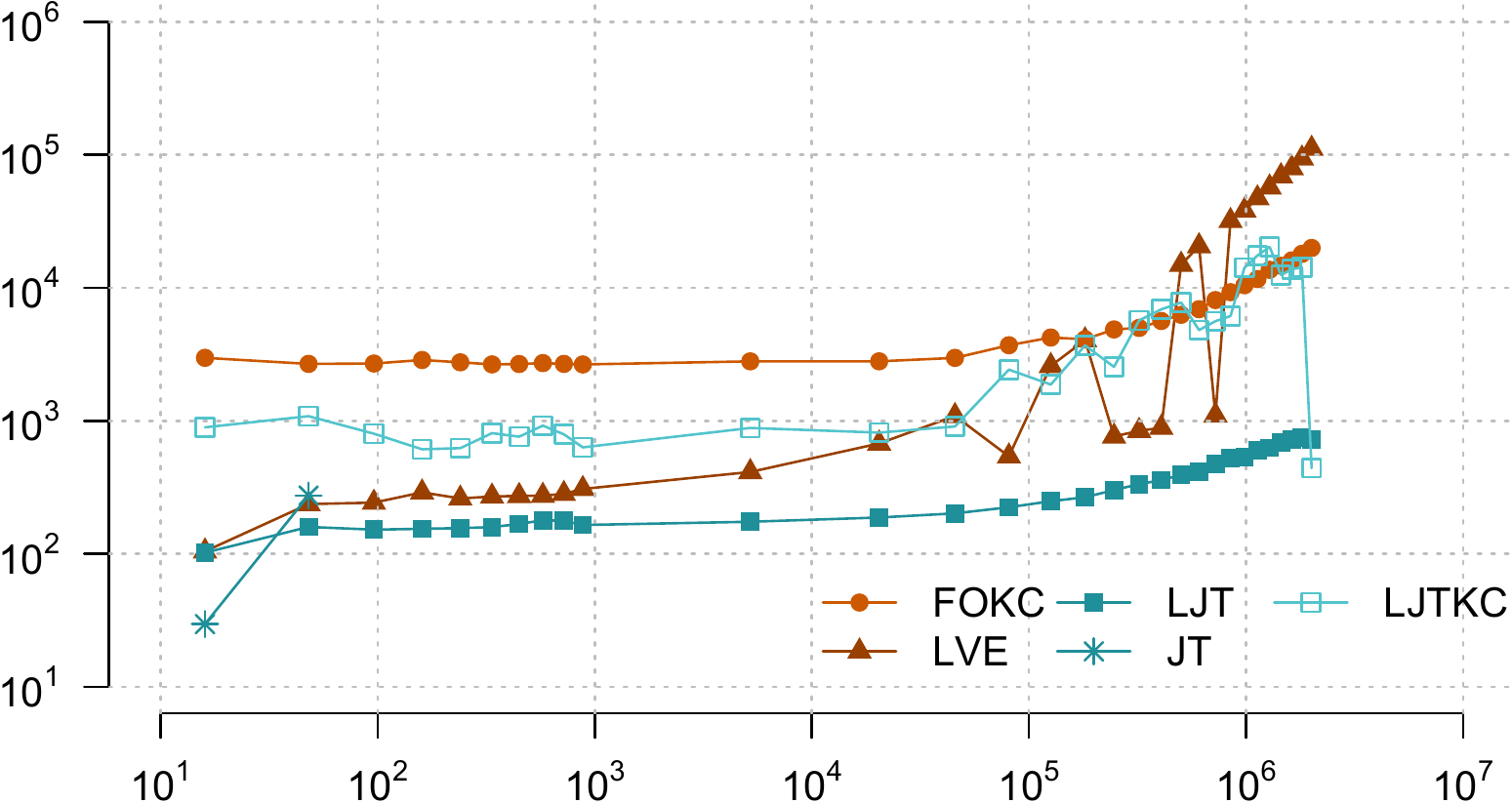}
\caption{Runtimes [ms] for $G_{ex}'$; on x-axis: $|gr(G_{ex}')|$ from $16$ to $2{,}004{,}000$; both axes on log scale; points connected for readability}
\label{fig:evSeq}
\end{minipage}\hfill
\begin{minipage}{.48\textwidth}
\centering
	\includegraphics[width=\textwidth]{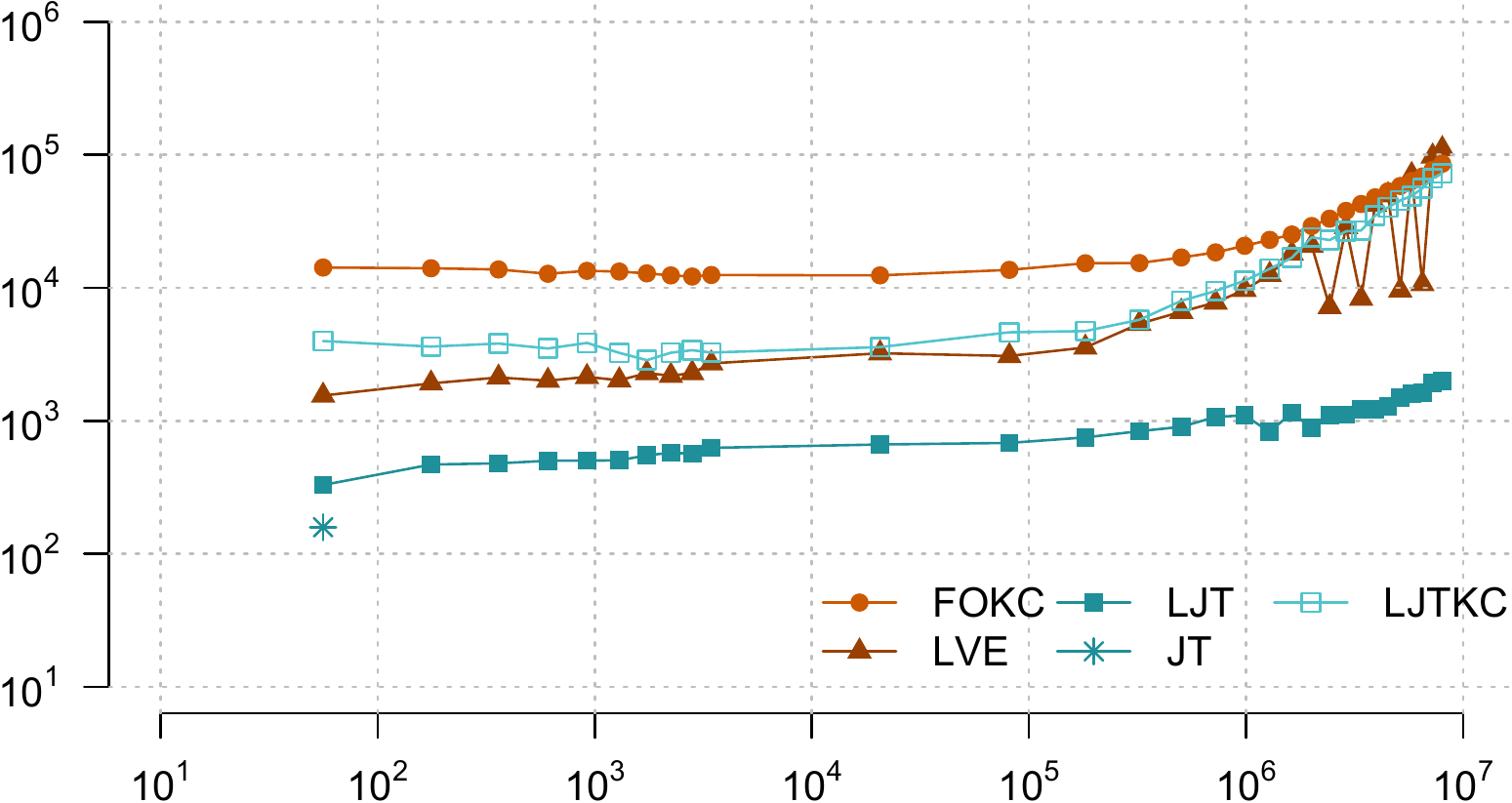}
\caption{Runtimes [ms] for $G_l'$; on x-axis: $|gr(G_l')|$ from $56$ to $8{,}012{,}000$; both axes on log scale; points connected for readability}
\label{fig:evLeq}
\end{minipage}
\end{figure*}

\section{Empirical Evaluation}
This evaluation demonstrates the speed up we can achieve for certain inputs when using LJT and FOKC in conjunction.
We have implemented a prototype of LJT, named \texttt{ljt} here.
Taghipour provides an implementation of LVE including its operators (available at \url{https://dtai.cs.kuleuven.be/software/gcfove}), named \texttt{lve}.
\cite{Bro13} provides an implementation of FOKC (available at \url{https://dtai.cs.kuleuven.be/software/wfomc}), named \texttt{fokc}.
For this paper, we integrated \texttt{fokc} into \texttt{ljt} to compute marginals at parclusters, named \texttt{ljtkc}.
Unfortunately, the FOKC implementation does not handle evidence in a lifted manner as described in \cite{BroDa12}.
Therefore, we do not consider evidence as \texttt{fokc} runtimes explode.
We have also implemented the propositional junction tree algorithm, named \texttt{jt}.

This evaluation has two parts:
First, we test two input models with inequalities to highlight
\begin{inparaenum}[(i)]
	\item how runtimes of LVE and, subsequently, LJT explode,
	\item how FOKC handles the inputs without the blowup in runtime, and
	\item how LJTKC provides a speedup for those inputs.
\end{inparaenum}
Second, we test two inputs without inequalities to highlight
\begin{inparaenum}[(i)]
	\item how runtimes of LVE and LJT compare to FOKC without inequalities and
	\item how LJT enables a fast and stable reasoning.
\end{inparaenum}
We compare overall runtimes without input parsing averaged over five runs with a working memory of 16GB.
%Runtimes of \texttt{ljt}, \texttt{fokc}, and \texttt{ljtkc} include all steps (FO jtree construction and message passing as well as compilation where applicable).
\texttt{lve} eliminates all non-query randvars from its input model for each query, grounding in the process.
\texttt{ljt} builds an FO jtree for its input model, passes messages, and then answers queries on submodels.
\texttt{fokc} forms a WFOMC problem for its input model, compiles a model circuit, compiles for each query a query circuit, and computes the marginals of all PRVs in the input model with random groundings.
\texttt{ljtkc} starts like \texttt{ljt} for its input model until answering queries.
It then calls \texttt{fokc} at each parcluster to compute marginals of parcluster PRVs with random groundings.
\texttt{jt} receives the grounded input models and otherwise proceeds like \texttt{ljt}.

\subsubsection{Inputs with Inequalities}
For the first part of this evaluation, we test two input models, $G_{ex}$ and a slightly larger model $G_l$ that is an extension of $G_{ex}$.
$G'$ has two more logvars, each with its own domain, and eight additional PRVs with one or two parameters.
The PRVs are arguments to twenty parfactors, each parfactor with one to three inputs.
The FO jtree for $G_l$ has six parclusters, the largest one containing five PRVs.
%We use $G_{ex}$ as input.
We vary the domain sizes from $2$ to $1000$, resulting in $|gr(G_{ex})|$ from $12$ to $2{,}002{,}000$ and $|gr(G_l)|$ from $52$ to $8{,}010{,}000$.
We query each PRV with random groundings, leading to $4$ and $12$ queries, respectively.
For $G_{ex}$, the queries could be
\begin{itemize}
	\item $Smokes(p_1)$,
	\item $Friends(p_1, p_2)$,
	\item $Asthma(p_1)$, and
	\item $Cancer(p_1)$,
\end{itemize}
where $p_i$ stands for a domain value of $X$ and $Y$.
\Cref{fig:evS,fig:evL} show for $G_{ex}$ and $G_l$ respectively runtimes in milliseconds [ms] with increasing $|gr(G)|$ on log-scaled axes, marked as follows:
\begin{itemize}
	\item \texttt{fokc}: circle, orange,
	\item \texttt{jt}: star, turquoise,
	\item \texttt{ljt}: filled square, turquoise,
	\item \texttt{ljtkc}: hollow square, light turquoise, and
	\item \texttt{lve}: triangle, dark orange.
\end{itemize}

In \cref{fig:evS}, we compare runtimes on the smaller model, $G_{ex}$, with four queries.
For the first two settings, \texttt{jt} is the fastest with a runtime of under $20$ms, while \texttt{fokc} is the slowest with over $2.700$ms.
After the fourth setting, the \texttt{jt} runtime explodes even more and memory errors occur.
\texttt{lve} and \texttt{ljt} have shorter runtimes than \texttt{fokc} and \texttt{ljtkc} for the first three settings as well, with \texttt{ljt} being faster than \texttt{lve} due to the smaller submodels for QA.
But, runtimes of \texttt{lve} and \texttt{ljt} steadily increase as the groundings become more severe with larger domain sizes.
With the seventh setting, both programs have memory errors.
\texttt{fokc} and \texttt{ljtkc} show runtimes that increase linearly with domain sizes.
Given this small model, \texttt{ljtkc} has minimally faster runtimes than \texttt{fokc}.

For the larger model, $G_l$, the runtime behaviour is similar as shown in \cref{fig:evL}.
Due to the larger model, the \texttt{jt} runtimes are already much longer with the first setting than the other runtimes.
Again, up to the third setting, \texttt{lve} and \texttt{ljt} perform better than \texttt{fokc} with \texttt{ljt} being faster than \texttt{lve} and from the seventh setting on, memory errors occur.
\texttt{ljtkc} performs best from the third setting onwards.
\texttt{ljtkc} and \texttt{fokc} show the same steady increase in runtimes as before.
\texttt{ljtkc} runtimes have a speedup of a factor from $0.13$ to $0.76$ for $G_l$ compared to \texttt{fokc}.
Up to a domain size of $100$ ($|gr(G_l)| = 81{,}000$), \texttt{ljtkc} saves around one order of magnitude.

For small domain sizes, \texttt{ljtkc} and \texttt{fokc} perform worst.
With increasing domain sizes, they outperform the other programs.
While not a part of this evaluation, experiments showed that with an increasing number of parfactors, \texttt{ljtkc} promises to outperform \texttt{fokc} even more, especially with smaller domain sizes (for our setups, $6$ to $500$).

\subsubsection{Inputs without Inequalities}
For the second part of this evaluation, we test two input models, $G_{ex}'$ and $G_l'$, that are both the models from the first part but with $Y$ receiving an own domain as large as $X$, making the inequality superfluous.
Domain sizes vary from $2$ to $1000$, resulting in $|gr(G_{ex}')|$ from $16$ to $2{,}004{,}000$ and $|gr(G_l')|$ from $56$ to $8{,}012{,}000$.
Each PRV is a query with random groundings again (without a $Y$ grounding).
\Cref{fig:evSeq,fig:evLeq} show for $G_{ex}'$ and $G_l'$ respectively runtimes in milliseconds [ms] with increasing $|gr(G)|$, marked as before.
Both axes are log-scaled.
Points are connected for readability.
%\begin{itemize}
%	\item \texttt{fokc}: circle, orange,
%	\item \texttt{jt}: star, turquoise,
%	\item \texttt{ljt}: filled square, turquoise,
%	\item \texttt{ljtkc}: hollow square, light turquoise, and
%	\item \texttt{lve}: triangle, dark orange.
%\end{itemize}

\Cref{fig:evSeq,fig:evLeq} show that \texttt{lve} and \texttt{ljt} do not exhibit the runtime explosion without inequalities.
\texttt{ljtkc} does not perform best as the overhead introduced by FOKC does not pay off as much.
In fact, \texttt{ljt} performs best in almost all cases.
In both figures, \texttt{jt} is the fastest for the first setting.
With the following settings, \texttt{jt} runs into memory problems while runtimes explode.
\texttt{lve} has a steadily increasing runtime for most parts, though a few settings lead to shorter runtimes with higher domain sizes.
We could not find an explanation for the decrease in runtime for those handful of settings.
Overall, \texttt{lve} runtimes rise more than the other runtimes apart from \texttt{jt}.
\texttt{ljtkc} exhibits an unsteady runtime performance on the smaller model, though again, we could not find an explanation for the jumps between various sizes.
With the larger model, \texttt{ljtkc} shows a more steady performance that is better than the one of \texttt{fokc}.
\texttt{ljtkc} is a factor of $0.2$ to $0.8$ faster.
\texttt{fokc} and \texttt{ljt} runtimes steadily increase with rising $|gr(G)|$.
\texttt{ljt} gains over an order of magnitude compared to \texttt{fokc}.
In the larger model, \texttt{ljt} is a factor of $0.02$ to $0.06$ than \texttt{fokc} over all domain sizes.

In summary, without inequalities \texttt{ljt} performs best on our input models, being faster by over an order of magnitude compared to \texttt{fokc}.
Though, \texttt{ljtkc} does not perform worst, \texttt{ljt} performs better and steadier.
With inequalities, \texttt{ljtkc} shows promise in speeding up performance.

\section{Conclusion}
We present a combination of FOKC and LJT to speed up inference.
For certain inputs, LJT (with LVE as a subroutine) and FOKC start to struggle either due to model structure or size.
LJT provides a means to cluster a model into submodels, on which any exact lifted inference algorithm can answer queries given the algorithm can handle evidence and messages in a lifted way.
FOKC fused with LJT and LVE can handle larger models more easily.
In turn, FOKC boosts LJT by avoiding groundings in certain cases.
The fused algorithm enables us to compute answers faster than LJT with LVE for certain inputs and LVE and FOKC alone.
%One may even plug in any probabilistic inference algorithm into LJT for query answering and, given it can handle conjunctive queries, message passing.

We currently work on incorporating FOKC into message passing for cases where an problematic elimination occurs during message calculation, which includes adapting an FO jtree accordingly.
We also work on learning lifted models to use as inputs for LJT.
Moreover, we look into constraint handling, possibly realising it with answer-set programming.
Other interesting algorithm features include parallelisation and caching as a means to speed up runtime.

\bibliographystyle{aaai}
\bibliography{lifted_inf}

\begin{thebibliography}{}

\bibitem[\protect\citeauthoryear{Ahmadi \bgroup et al\mbox.\egroup
  }{2013}]{AhmKeMlNa13}
Ahmadi, B.; Kersting, K.; Mladenov, M.; and Natarajan, S.
\newblock 2013.
\newblock {Exploiting Symmetries for Scaling Loopy Belief Propagation and
  Relational Training}.
\newblock {\em Machine Learning} 92(1):91--132.

\bibitem[\protect\citeauthoryear{Apsel and Brafman}{2011}]{ApsBr11}
Apsel, U., and Brafman, R.~I.
\newblock 2011.
\newblock {Extended Lifted Inference with Joint Formulas}.
\newblock In {\em UAI-11 Proceedings of the 27th Conference on Uncertainty in
  Artificial Intelligence}.

\bibitem[\protect\citeauthoryear{Bellodi \bgroup et al\mbox.\egroup
  }{2014}]{BelLaRiCoZe14}
Bellodi, E.; Lamma, E.; Riguzzi, F.; Costa, V.~S.; and Zese, R.
\newblock 2014.
\newblock {Lifted Variable Elimination for Probabilistic Logic Programming}.
\newblock {\em Theory and Practice of Logic Programming} 14(4--5):681--695.

\bibitem[\protect\citeauthoryear{Braun and M\"oller}{2016}]{BraMo16a}
Braun, T., and M\"oller, R.
\newblock 2016.
\newblock {Lifted Junction Tree Algorithm}.
\newblock In {\em Proceedings of {KI} 2016: Advances in Artificial
  Intelligence},  30--42.
\newblock Springer.

\bibitem[\protect\citeauthoryear{Braun and M\"oller}{2018}]{BraMo18a}
Braun, T., and M\"oller, R.
\newblock 2018.
\newblock {Lifted Most Probable Explanation}.
\newblock In {\em Proceedings of the International Conference on Conceptual
  Structures},  39--54.
\newblock Springer.

\bibitem[\protect\citeauthoryear{Chavira and Darwiche}{2007}]{ChaDa07}
Chavira, M., and Darwiche, A.
\newblock 2007.
\newblock {Compiling Bayesian Networks Using Variable Elimination}.
\newblock In {\em IJCAI-07 Proceedings of the 20th International Joint
  Conference on Artificial Intelligence},  2443--2449.

\bibitem[\protect\citeauthoryear{Chavira and Darwiche}{2008}]{ChaDa08}
Chavira, M., and Darwiche, A.
\newblock 2008.
\newblock {On Probabilistic Inference by Weighted Model Counting}.
\newblock {\em Artificial Intelligence} 172(6-7):772--799.

\bibitem[\protect\citeauthoryear{Choi, Amir, and Hill}{2010}]{ChoiAmHi12}
Choi, J.; Amir, E.; and Hill, D.~J.
\newblock 2010.
\newblock {Lifted Inference for Relational Continuous Models}.
\newblock In {\em UAI-10 Proceedings of the 26th Conference on Uncertainty in
  Artificial Intelligence},  13--18.

\bibitem[\protect\citeauthoryear{Darwiche and Marquis}{2002}]{DarMa02}
Darwiche, A., and Marquis, P.
\newblock 2002.
\newblock {A Knowledge Compilation Map}.
\newblock {\em Journal of Artificial Intelligence Research} 17(1):229--264.

\bibitem[\protect\citeauthoryear{Das \bgroup et al\mbox.\egroup
  }{2016}]{DaWuKhKeNa16}
Das, M.; Wu, Y.; Khot, T.; Kersting, K.; and Natarajan, S.
\newblock 2016.
\newblock {Scaling Lifted Probabilistic Inference and Learning Via Graph
  Databases}.
\newblock In {\em Proceedings of the SIAM International Conference on Data
  Mining},  738--746.

\bibitem[\protect\citeauthoryear{de Salvo~Braz}{2007}]{Braz07}
de~Salvo~Braz, R.
\newblock 2007.
\newblock {\em {Lifted First-order Probabilistic Inference}}.
\newblock Ph.D. Dissertation, University of Illinois at Urbana Champaign.

\bibitem[\protect\citeauthoryear{Gogate and Domingos}{2010}]{GogDo10}
Gogate, V., and Domingos, P.
\newblock 2010.
\newblock {Exploiting Logical Structure in Lifted Probabilistic Inference}.
\newblock In {\em Working Note of the Workshop on Statistical Relational
  Artificial Intelligence at the 24th Conference on Artificial Intelligence},
  19--25.

\bibitem[\protect\citeauthoryear{Gogate and Domingos}{2011}]{GogDo11}
Gogate, V., and Domingos, P.
\newblock 2011.
\newblock {Probabilistic Theorem Proving}.
\newblock In {\em UAI-11 Proceedings of the 27th Conference on Uncertainty in
  Artificial Intelligence},  256--265.

\bibitem[\protect\citeauthoryear{Kazemi and Poole}{2016}]{KazPo16}
Kazemi, S.~M., and Poole, D.
\newblock 2016.
\newblock {Why is Compiling Lifted Inference into a Low-Level Language so
  Effective?}
\newblock In {\em IJCAI-16 Statistical Relational AI Workshop}.

\bibitem[\protect\citeauthoryear{Lauritzen and Spiegelhalter}{1988}]{LauSp88}
Lauritzen, S.~L., and Spiegelhalter, D.~J.
\newblock 1988.
\newblock {Local Computations with Probabilities on Graphical Structures and
  Their Application to Expert Systems}.
\newblock {\em Journal of the Royal Statistical Society. Series B:
  Methodological} 50:157--224.

\bibitem[\protect\citeauthoryear{Milch \bgroup et al\mbox.\egroup
  }{2008}]{MilZeHaKa08}
Milch, B.; Zettelmoyer, L.~S.; Kersting, K.; Haimes, M.; and Kaelbling, L.~P.
\newblock 2008.
\newblock {Lifted Probabilistic Inference with Counting Formulas}.
\newblock In {\em AAAI-08 Proceedings of the 23rd Conference on Artificial
  Intelligence},  1062--1068.

\bibitem[\protect\citeauthoryear{Poole and Zhang}{2003}]{PooZh03}
Poole, D., and Zhang, N.~L.
\newblock 2003.
\newblock {Exploiting Contextual Independence in Probabilistic Inference}.
\newblock {\em Jounal of Artificial Intelligence} 18:263--313.

\bibitem[\protect\citeauthoryear{Poole}{2003}]{Poo03}
Poole, D.
\newblock 2003.
\newblock {First-order Probabilistic Inference}.
\newblock In {\em IJCAI-03 Proceedings of the 18th International Joint
  Conference on Artificial Intelligence}.

\bibitem[\protect\citeauthoryear{Shenoy and Shafer}{1990}]{SheSh90}
Shenoy, P.~P., and Shafer, G.~R.
\newblock 1990.
\newblock {Axioms for Probability and Belief-Function Propagation}.
\newblock {\em Uncertainty in Artificial Intelligence 4} 9:169--198.

\bibitem[\protect\citeauthoryear{Singla and Domingos}{2008}]{SinDo08}
Singla, P., and Domingos, P.
\newblock 2008.
\newblock {Lifted First-order Belief Propagation}.
\newblock In {\em AAAI-08 Proceedings of the 23rd Conference on Artificial
  Intelligence},  1094--1099.

\bibitem[\protect\citeauthoryear{Taghipour and Davis}{2012}]{TagDa12}
Taghipour, N., and Davis, J.
\newblock 2012.
\newblock {Generalized Counting for Lifted Variable Elimination}.
\newblock In {\em Proceedings of the 2nd International Workshop on Statistical
  Relational AI},  1--8.

\bibitem[\protect\citeauthoryear{Taghipour \bgroup et al\mbox.\egroup
  }{2013}]{TagFiDaBl13}
Taghipour, N.; Fierens, D.; Davis, J.; and Blockeel, H.
\newblock 2013.
\newblock {Lifted Variable Elimination: Decoupling the Operators from the
  Constraint Language}.
\newblock {\em Journal of Artificial Intelligence Research} 47(1):393--439.

\bibitem[\protect\citeauthoryear{van~den Broeck and Davis}{2012}]{BroDa12}
van~den Broeck, G., and Davis, J.
\newblock 2012.
\newblock {Conditioning in First-Order Knowledge Compilation and Lifted
  Probabilistic Inference}.
\newblock In {\em Proceedings of the 26th AAAI Conference on Artificial
  Intelligence},  1961--1967.

\bibitem[\protect\citeauthoryear{van~den Broeck and Niepert}{2015}]{BroNi15}
van~den Broeck, G., and Niepert, M.
\newblock 2015.
\newblock {Lifted Probabilistic Inference for Asymmetric Graphical Models}.
\newblock In {\em AAAI-15 Proceedings of the 29th Conference on Artificial
  Intelligence},  3599--3605.

\bibitem[\protect\citeauthoryear{van~den Broeck \bgroup et al\mbox.\egroup
  }{2011}]{BroTaMeDaRa11}
van~den Broeck, G.; Taghipour, N.; Meert, W.; Davis, J.; and Raedt, L.~D.
\newblock 2011.
\newblock {Lifted Probabilistic Inference by First-order Knowledge
  Compilation}.
\newblock In {\em IJCAI-11 Proceedings of the 22nd International Joint
  Conference on Artificial Intelligence}.

\bibitem[\protect\citeauthoryear{van~den Broeck}{2013}]{Bro13}
van~den Broeck, G.
\newblock 2013.
\newblock {\em {Lifted Inference and Learning in Statistical Relational
  Models}}.
\newblock Ph.D. Dissertation, KU Leuven.

\bibitem[\protect\citeauthoryear{Vlasselaer \bgroup et al\mbox.\egroup
  }{2016}]{VlaMeBrRa16}
Vlasselaer, J.; Meert, W.; van~den Broeck, G.; and Raedt, L.~D.
\newblock 2016.
\newblock {Exploiting Local and Repeated Structure in Dynamic Baysian
  Networks}.
\newblock {\em Artificial Intelligence} 232:43--53.

\bibitem[\protect\citeauthoryear{Zhang and Poole}{1994}]{ZhaPo94}
Zhang, N.~L., and Poole, D.
\newblock 1994.
\newblock {A Simple Approach to Bayesian Network Computations}.
\newblock In {\em Proceedings of the 10th Canadian Conference on Artificial
  Intelligence},  171--178.

\end{thebibliography}

\end{document}